\documentclass[10pt]{article}
\usepackage{float}
\usepackage{amsmath}
\usepackage{amsthm}
\usepackage{amssymb}
\usepackage{graphicx} 
\usepackage{subfigure}

\usepackage{algorithm}
\usepackage{algorithmic}
\usepackage{comment}

\usepackage{hyperref}

\newcommand{\echo}[1]{\textnormal{#1}}

\usepackage{times}
\usepackage{multirow}
\usepackage{proceed2e}

\newtheorem{lemma}{Lemma}
\newtheorem{corollary}{Corollary}
\newtheorem{theorem}{Theorem}

\newcommand{\argmin}{\mathop{\arg\min}}
\newcommand{\tr}{\mathop{\mathrm{tr}}}

\floatstyle{ruled}
\newfloat{algorithm}{tbp}{loa}
\providecommand{\algorithmname}{Algorithm}
\floatname{algorithm}{\protect\algorithmname}

\title{Normalized online learning}
\author{ {\bf St\'ephane Ross} \\  
Carnegie Mellon University\\ 
Pittsburgh, PA, USA\\ 
\texttt{stephaneross@cmu.edu}\\
\And 
{\bf Paul Mineiro}  \\ 
Microsoft \\ 
Bellevue, WA, USA \\     
\texttt{paul.mineiro@gmail.com}\\         
\And 
{\bf John Langford}   \\ 
Microsoft Research \\          
New York, NY, USA   \\           
\texttt{jcl@microsoft.com}\\
} 

\begin{document}

\maketitle

\begin{abstract}
  We introduce online learning algorithms which are
  independent of feature scales, proving regret bounds dependent on
  the ratio of scales existent in the data rather than the absolute
  scale. This has several useful effects: there is no need to
  pre-normalize data, the test-time and test-space complexity are
  reduced, and the algorithms are more robust.
\end{abstract}

\section{Introduction}

Any learning algorithm can be made invariant by initially transforming
all data to a preferred coordinate system.  In practice many
algorithms begin by applying an affine transform to features so they
are zero mean with standard deviation 1~\cite{li1998sphering}.  For
large data sets in the batch setting this preprocessing can be
expensive, and in the online setting the analogous operation is
unclear.  Furthermore preprocessing is not applicable if the inputs to
the algorithm are generated dynamically during learning, e.g., from an
on-demand primal representation of a kernel~\cite{SonFra10}, virtual
example generated to enforce an
invariant~\cite{loosli-canu-bottou-2006}, or machine learning
reduction~\cite{allwein2001reducing}.

When normalization techniques are too expensive or impossible we can
either accept a loss of performance due to the use of misnormalized
data or design learning algorithms which are inherently capable of
dealing with unnormalized data.  In the field of optimization, it is a
settled matter that algorithms should operate independent of
an individual dimensions scaling~\cite{oren1974}.  The same structure
defines natural gradients~\cite{Daw98} where in the stochastic
setting, results indicate that for the parametric case the Fisher
metric is the unique invariant metric satisfying a certain regular and
monotone property~\cite{cG98}.  Our interest here is in the online
learning setting, where this structure is rare: typically regret
bounds depend on the norm of features.

The biggest practical benefit of invariance to feature scaling is that
learning algorithms ``just work'' in a more general sense.  This is of
significant importance in online learning settings where fiddling with
hyper-parameters is often common, and this work can be regarded as an
alternative to investigations of optimal hyper-parameter
tuning~\cite{BB12,SLA12,HHL13}.  With a normalized update users do not need to
know (or remember) to pre-normalize input datasets and the need to
worry about hyper-parameter tuning is greatly reduced.  In practical
experience, it is common for those unfamiliar with machine learning to
create and attempt to use datasets without proper normalization.

Eliminating the need to normalize data also reduces computational
requirements at both training and test time.  For particularly large
datasets this can become important, since the computational cost in
time and RAM of doing normalization can rival the cost and time of
doing the machine learning (or even worse for naive centering of
sparse data).  Similarly, for applications which are constrained by
testing time, knocking out the need for feature normalization allows
more computational performance with the same features or better
prediction performance when using the freed computational resources to
use more features.

\subsection{Adversarial Scaling}
Adversarial analysis is fairly standard in online learning.  However,
an adversary capable of rescaling features can induce unbounded regret
in common gradient descent methods.  As
an example consider the standard regret bound for projected online
convex subgradient descent after $T$ rounds using the best learning
rate in hindsight~\cite{Z3},
\begin{displaymath}
R \leq \sqrt{T} ||w^*||_2 \max_{t \in 1:T} ||g_t||_2 .
\end{displaymath}
Here $w^*$ is the best predictor in hindsight and $\{ g_t \}$ is the
sequence of instantaneous gradients encountered by the
algorithm. Suppose $w^* = (1, 1) \in \mathbb{R}^2$ and imagine scaling
the first coordinate by a factor of $s$.  As $s \to \infty$,
$||w^*||_2$ approaches 1, but unfortunately for a linear predictor the
gradient is proportional to the input, so $\max_{t \in 1:T} ||g_t||_2$
can be made arbitrarily large.  Conversely as $s \to 0$, the gradient
sequence remains bounded but $||w^*||_2$ becomes arbitrarily large.
In both cases the regret bound can be made arbitrarily poor.  This is
a real effect rather than a mere artifact of analysis, as indicated by experiments with a synthetic two dimensional dataset in figure~\ref{fig:toy}.
\begin{figure}\label{fig:toy}
\begin{center}
\includegraphics[angle=270,scale=0.6]{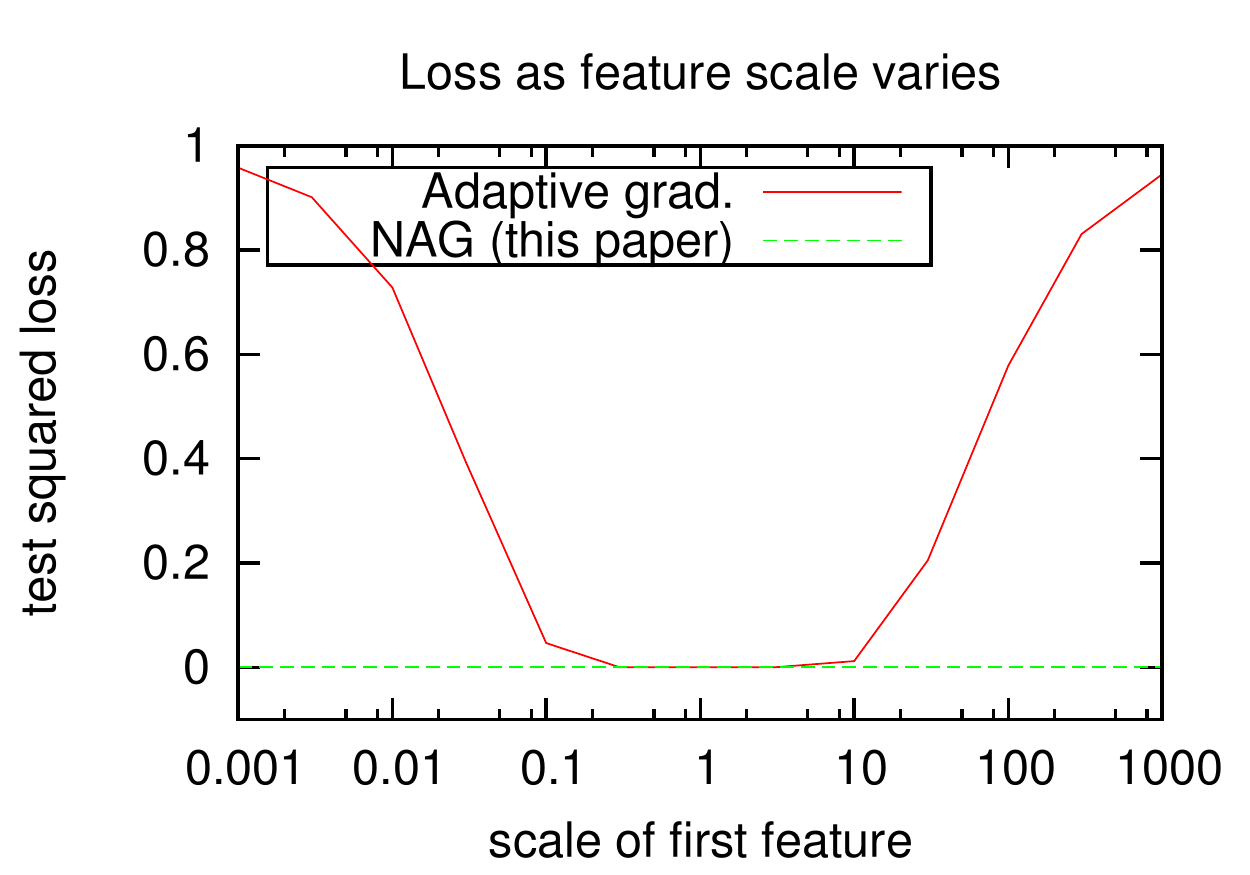}
\end{center}
\vspace{-16pt}
\caption{A comparison of performance of NAG (this paper) and adaptive gradient~\cite{MS10,DHS11} on a synthetic dataset with varying scale in the first feature.}
\end{figure}

Adaptive first-order online methods~\cite{MS10,DHS11} also have this 
vulnerability, despite adapting the geometry to the input sequence.  Consider 
a variant of the adaptive gradient update (without projection)
\begin{displaymath}
w_{t+1} = w_t - \eta\, \mathrm{diag} (\sum_{s=1}^t g_s g_s^T)^{-1/2} g_t,
\end{displaymath}
which has associated regret bound of order
\begin{displaymath}
||w^*||_2 \, d^{1/2} \sqrt{\inf_S \left\{ \sum_{t=1}^T \langle g_t, S^{-1} g_t \rangle : S \succeq 0, \mathrm{tr} (S) \leq d \right\} } .
\end{displaymath}
Again by manipulating the scaling of a single axis this can be made
arbitrarily poor.

The online Newton step~\cite{H6} algorithm has a regret bound
independent of units as we address here.  Unfortunately ONS space and
time complexity grows quadratically with the length of the input
sequence, but the existence of ONS motivates the search for
computationally viable scale invariant online learning rules.

Similarly, the second order perceptron~\cite{SO} and AROW~\cite{AROW}
partially address this problem for hinge loss.  These algorithms are
not unit-free because they have hyperparameters whose optimal value
varies with the scaling of features and again have running times that
are superlinear in the dimensionality.  More recently, diagonalized
second order perceptron and AROW have been proposed~\cite{SOROW}.
These algorithms are linear time, but their analysis is generally not
unit free since it explicitly depends on the norm of the weight
vector.  \echo{Corollary 3 is unit invariant.  A comparative analysis
  of empirical performance would be interesting to observe.}

\echo{The use of unit invariant updates have been implicitly studied
  with asymptotic analysis and empirics.  For example~\cite{SZL12}
  uses a per-parameter learning rate proportional to an estimate of
  gradient squared divided by variance and second derivative.
  Relative to this work, we prove finite regret bound guarantees for
  our algorithm.}

\subsection{Contributions}

We define normalized online learning algorithms which are invariant to
feature scaling, then show that these are interesting algorithms
theoretically and experimentally.

We define a scaling adversary for online learning analysis.  The
critical additional property of this adversary is that algorithms with
bounded regret must have updates which are invariant to feature scale.
We prove that our algorithm has a small regret against this more
stringent adversary.

We then experiment with this learning algorithm on a number of
datasets.  For pre-normalized datasets, we find that it makes little
difference as expected, while for unnormalized or improperly
normalized datasets this update rule offers large advantages over
standard online update rules.  \echo{All of our code is a part of the open
source Vowpal Wabbit project~\cite{VW12}.}

\section{Notation}

Throughout this draft, the indices $i,j$ indicate elements of a
vector, while the index $t,T$ or a particular number indicates time.
A label $y$ is associated with some features $x$, and we are concerned
with linear prediction $\sum_i w_i x_i$ resulting in some loss for
which a gradient $g$ can be computed with respect to the weights.
Other notation is introduced as it is defined.  

\section{The algorithm}

We start with the simplest version of a scale invariant online learning
algorithm.
\begin{algorithm}
\caption{\label{alg:NG}NG(learning\_rate $\eta_t$)}

\begin{enumerate}
\item Initially $w_{i}=0$, $s_{i}=0$, $N=0$

\item For each timestep $t$ observe example $(x,y)$

\begin{enumerate}
\item For each $i$, if $|x_{i}|>s_{i}$

\begin{enumerate}
\item $w_{i}\leftarrow\frac{w_{i}s_{i}^2}{|x_{i}|^2}$
\item $s_{i}\leftarrow|x_{i}|$
\end{enumerate}
\item $\hat{y}=\sum_{i}w_{i}x_{i}$
\item $N \leftarrow N + \sum_i \frac{x_i^2}{s_i^2}$
\item For each $i$,
\begin{enumerate}
\item $w_{i}\leftarrow w_{i}-\eta_t \frac{t}{N} \frac{1}{s_{i}^2}\frac{\partial L(\hat{y},y)}{\partial w_{i}}$
\end{enumerate}
\end{enumerate}
\end{enumerate}
\end{algorithm}

NG (Normalized Gradient Descent) is presented in algorithm
\ref{alg:NG}. NG adds scale invariance to online gradient descent,
making it work for any scaling of features within the dataset.

Without $s,N$, this algorithm simplifies to standard stochastic
gradient descent.

The vector element $s_{i}$ stores the magnitude of feature $i$
according to \echo{$s_{ti}=\max_{t'\in\{1...t\}}|x_{t'i}|$}. These are
updated and maintained online in steps 2.(a).ii, and used to rescale
the update on a per-feature basis in step 2.(d).i.

Using $N$ makes the learning rate (rather than feature scale) control
the average change in prediction from an update.  Here $N/t$ is the
average change in the prediction excluding $\eta$, so multiplying by
$1/(N/t) = t/N$ causes the average change in the prediction to be entirely
controlled by $\eta$.

Step 2.(a).i squashes a weight $i$ when a new scale is encountered.
Neglecting the impact of $N$, the new value is precisely equal to what
the weight's value would have been if all previous updates used the
new scale.

Many other online learning algorithms can be made scale invariant
using variants of this approach.  One attractive choice is adaptive
gradient descent~\cite{MS10,DHS11} since this also has per-feature learning rates.  The normalized version of adaptive gradient descent is given in algorithm~\ref{alg:NAG}.

In order to use this, the algorithm must maintain the sum of gradients
squared $G_{i}=\sum_{(x,y) \mbox{ observed}}\left(\frac{\partial
  L(\hat{y},y)}{\partial w_{i}}\right)^{2}$ for feature $i$ in step
2.d.i.  The interaction between $N$ and $G$ is somewhat tricky,
because a large average update (i.e. most features have a magnitude
near their scale) increases the value of $G_i$ as well as $N$ implying
the power on $N$ must be decreased to compensate.  Similarly, we
reduce the power on $s_i$ and $|x_i|$ to $1$ throughout.  The more
complex update rule is scale invariant and the dependence on $N$
introduces an automatic global rescaling of the update rule.

\begin{algorithm}
\caption{\label{alg:NAG}NAG(learning\_rate $\eta$)}

\begin{enumerate}
\item Initially $w_{i}=0$, $s_{i}=0$, $G_{i}=0$, $N=0$

\item For each timestep $t$ observe example $(x,y)$ 

\begin{enumerate}
\item For each $i$, if $|x_{i}|>s_{i}$ 

\begin{enumerate}
\item $w_{i}\leftarrow\frac{w_{i}s_{i}}{|x_{i}|}$
\item $s_{i}\leftarrow|x_{i}|$
\end{enumerate}
\item $\hat{y}=\sum_{i}w_{i}x_{i}$
\item $N \leftarrow N + \sum_i \frac{x_i^2}{s_i^2}$
\item For each $i$, 
\begin{enumerate}
\item $G_i \leftarrow G_i + \left(\frac{\partial L(\hat{y},y)}{\partial w_{i}}\right)^2$
\item $w_{i}\leftarrow w_{i}-\eta \sqrt{\frac{t}{N}} \frac{1}{s_{i} \sqrt{G_i}}\frac{\partial L(\hat{y},y)}{\partial w_{i}}$\end{enumerate}
\end{enumerate}
\end{enumerate}
\end{algorithm}

In the next sections we analyze and justify this algorithm.  We demonstrate that NAG competes well against a set of predictors $w$ with predictions ($w^\top x$) bounded by some constant over all the inputs $x_t$ seen during training. In practice, as this is potentially sensitive to outliers, we also consider a squared norm version of NAG, which we refer to as sNAG that is a straightforward modification---we simply keep the accumulator $s_i =
\sum x_i^2$ and use $\sqrt{s_i/t}$ in the update rule. That is, normalization is carried using the standard deviation (more precisely, the square root of the second moment) of each feature, rather than the max norm. With respect to our analysis below, this simple modification can be interpreted as changing slightly the set of predictors we compete against, i.e. predictors with predictions bounded by a constant only over the inputs within 1 standard deviation. Intuitively, this is more robust and appropriate in the presence of outliers. While our analysis focuses on NAG, in practice, sNAG sometimes yield improved performance.

\section{The Scaling Adversary Setting}

In common machine learning practice, the choice of units for any
particular feature is arbitrary.  For example, when estimating the
value of a house, the land associated with a house may be encoded
either in acres or square feet.  To model this effect, we propose a
scaling adversary, which is more powerful than the standard adversary
in adversarial online learning settings.

The setting for analysis is similar to adversarial online linear
learning, with the primary difference in the goal.  The setting
proceeds in the following round-by-round fashion where
\begin{enumerate}
\item Prior to all rounds, the adversary commits to a fixed
positive-definite matrix $S$.  This is not revealed to the learner.
\item On each round $t$,
\begin{enumerate}
\item The adversary chooses a vector $x_t$ such that $|| S^{1/2} x_t
||_\infty \leq 1$, where $S^{1/2}$ is the principal square root.
\item The learner makes a prediction $\hat{y}_t = w_t^\top x_t$.
\item The correct label $y_t$ is revealed and a loss $\ell(\hat{y}_t,y_t)$
is incurred.
\item The learner updates the predictor to $w_{t+1}$.
\end{enumerate}
\end{enumerate}
For example, in a regression setting, $\ell$ could be the squared loss
$\ell(\hat{y},y) = (\hat{y}-y)^2$, or in a binary classification setting,
$\ell$ could be the hinge loss $\ell(\hat{y},y) = \max(0,1-y\hat{y})$.
We consider general cases where the loss is only a function of $\hat{y}$
(i.e. no direct penalty on $w$) and convex in $\hat{y}$ (therefore convex
in $w$).

Although step 1 above is phrased in terms of an adversary, in practice
what is being modeled is ``the data set was prepared using arbitrary
units for each feature.''

Step 2 (a) above is phrased in terms of $\infty$-norm for ease of
exposition, but more generally can be considered any $p$-norm. Additionally, this step can be generalized to impose a different constraint on the inputs. For instance, instead requiring all points lie inside some $p$-norm ball, we could require that the second moment of the inputs, under some scaling matrix $S$ is 1. This is the model of the adversary for sNAG. 

\subsection{Competing against a Bounded Output Predictor}

Our goal is to compete against the set of weight vectors whose output
is bounded by some constant $C$ over the set of inputs the adversary
can choose.  Given step 2 (a) above, this is equivalent to
$\mathcal{W}^C_X = \{w \mid ||S^{-1/2}w||_1 \leq C \}$, i.e., the set
of $w$ with dual norm less than $C$.  In other words, the regret $R_T$ at
timestep $T$ is given by:
$$ R_T = \sum_{t=1}^T \ell(\hat{y}_t, y_t) - \min_{w | \forall t, w^\top x_t \leq C} \sum_{t=1}^T \ell(w^\top x_t, y_t)
$$
\echo{Here we use the fact that $\{w^\top x_t \leq C \}  = \{w \mid ||S^{-1/2}w||_1 \leq C \}$.}
In the more general case of a $p$-norm for step 2 (a), we would choose
$\mathcal{W}^C_X = \{w \mid ||S^{-1/2}w||_q \leq C \}$ for $q$ such that
$\frac{1}{p} + \frac{1}{q} = 1$. Note that the ``true'' $S$ of the adversary is an abstraction. It is unknown and only partially revealed through the data. In our analysis, we will instead be interested to bound regret against bounded output predictors for an empirical estimate of $S$, defined by the minimum volume $L_p$ ball containing all observed inputs. For $p=\infty$, $\mathcal{W}^C_X$ for the ``true'' $S$ is always a subset of the predictors allowed under this empirical $S$ (assuming both are diagonal matrices). In general, this does not necessarily hold for all $p$ norms, but the empirical $S$ always allows a larger volume of predictors than the ``true'' $S$.

\section{Analysis}

In this section, we analyze scale invariant update rules in several
ways.  The analysis is structurally similar to that used for adaptive
gradient descent~\cite{MS10,DHS11} with necessary differences to
achieve scale invariance.  We analyze the best solution in hindsight,
the best solution in a transductive setting, and the best solution in
an online setting.  These settings are each a bit more difficult than
the previous, and in each we prove regret bounds which are invariant
to feature scales.

We consider algorithms updating according to $w_{t+1} = w_t - A^{-1}_t
g_t$, where $g_t$ is the gradient of the loss at time $t$ w.r.t. $w$
at $w_t$, and $A_t$ is a sequence of $d \times d$ symmetric positive
(semi-)definite matrices that our algorithm can choose.  Both algorithms~\ref{alg:NG} and~\ref{alg:NAG} fit this general framework.  Combining the convexity of the loss
function and the definition of the update rule yields the following
result.
\begin{lemma}\label{lemma:regretbound}
\begin{displaymath}
\begin{split}
2 R_T &\leq (w_1 - w^*) A_1 (w_1 - w^*) \\
&\;\;\; + \sum_{t=1}^T (w_t - w^*)^\top (A_{t+1}-A_t) (w_t - w^*) \\
&\;\;\; + \sum_{t=1}^T g_t^\top A^{-1}_t g_t.
\end{split}
\end{displaymath}
We defer all proofs to the appendix.
\end{lemma}

\subsection{Best Choice of Conditioner in Hindsight}

Suppose we start from $w_1 = 0$ and before the start of the algorithm,
we would try to guess what is the best fixed matrix $A$, so that $A_t
= A$ for all $t$. In order to minimize regret, what would the best
guess be? This was initially analyzed for adaptive gradient
descent~\cite{MS10,DHS11}.  Consider the case where $A$ is a diagonal
matrix.

Using lemma \ref{lemma:regretbound}, for a fixed diagonal matrix $A$
and with $w_1 = 0$, the regret bound is:
\begin{displaymath}
R_T \leq \frac{1}{2} \sum_{i=1}^d \left( A_{ii} (w^*_i)^2 + \sum_{t=1}^T \frac{g^2_{ti}}{A_{ii}} \right).
\end{displaymath}
Taking the derivative w.r.t. $A_{ii}$, we obtain:
\begin{displaymath}
\begin{split}
&\frac{\partial}{\partial A_{ii}} \frac{1}{2} \sum_{i=1}^d \left( A_{ii} (w^*_i)^2 + \sum_{t=1}^T \frac{g^2_{ti}}{A_{ii}} \right) \\
&\;\;\; = \frac{1}{2}\left( (w^*_i)^2 - \sum_{t=1}^T \frac{g^2_{ti}}{A^2_{ii}} \right).
\end{split}
\end{displaymath}
Solving for when this is 0, we obtain
\begin{displaymath}
A^*_{ii} = \frac{\sqrt{\sum_{t=1}^T g^2_{ti}}}{\vert w^*_i \vert}.
\end{displaymath}
For this particular choice of $A$, then the regret is bounded by
\begin{displaymath}
R_T \leq \sum_{i=1}^d \left( \vert w^*_i \vert \sqrt{\sum_{t=1}^T g^2_{ti}} \right).
\end{displaymath}
We can observe that this regret is the same no matter the scaling of
the inputs. For instance if any axis $i$ is scaled by a factor $s_i$,
then $w^*_i$ would be a factor $1/s_i$ smaller, and the gradient $g_{ti}$
a factor $s_i$ larger, which would cancel out. Hence this regret can be
thought as the regret the algorithm would obtain when all features have
the same unit scale.

However, because of the dependency of $A$ on $w^*$, this does not give
us a good way to approximate this with data we have observed so
far. To remove this dependency, we can analyze for the best $A$ when
assuming a worst case for $w^*$.  This is the point at which the
analysis here differs from adaptive gradient descent where the dependence
on $w^*$ was dropped.  

\begin{lemma}\label{lemma:minimaxhindsight}
Let $S$ be the diagonal matrix with minimum determinant (volume)
s.t. $||S^{1/2} x_i||_p \leq 1$ for all $i \in 1:T$.  The solution to
\begin{displaymath}
\min_{A} \max_{w^* \in \mathcal{W}^C_X} \frac{1}{2} \sum_{i=1}^d \left( A_{ii} (w^*_i)^2 + \sum_{t=1}^T \frac{g^2_{ti}}{A_{ii}} \right)
\end{displaymath}
is given by
\begin{displaymath}
A^*_{ii} = \frac{1}{C} \sqrt{\frac{\sum_{t=1}^T g^2_{ti}}{S_{ii}}},
\end{displaymath}
and the regret bound for this particular choice of $A$ is given by
\begin{displaymath}
R_T \leq C \sum_{i=1}^d \sqrt{S_{ii} \sum_{t=1}^T g^2_{ti}}.
\end{displaymath}
\end{lemma}

Again the value of the regret bound does not change if the features are
rescaled.  This is most easily appreciated by considering a specific norm.
The simplest case is for $p=\infty$ where the coefficients $S_{ii}$ can
be defined directly in terms of the range of each feature, i.e. $S_{ii}
= \frac{1}{\max_t |x_{ti}|^2}$.  Thus for $p=\infty$, we can choose
\begin{displaymath}
A^*_{ii} = \frac{1}{C} \sqrt{\sum_{t=1}^T g^2_{ti}} \max_{t\in1:T} |x_{ti}|,
\end{displaymath}
leading to a regret of
\begin{displaymath}
R_T \leq C \sum_{i=1}^d \frac{\sqrt{\sum_{t=1}^T g^2_{ti}}}{\max_{t\in1:T} |x_{ti}|}.
\end{displaymath}
The scale invariance of the regret bound is now readily apparent.
This regret can potentially be order $O(d\sqrt{T})$.

\subsection{\echo{$p=2$} case}
\label{sec:p2}
For $p=2$, computing the coefficients $S_{ii}$ is more complicated,
but if you have access to the actual coefficients $S_{ii}$, the regret
is order $O(\sqrt{dT})$. This can be seen as follows. Let $g'_t =
\left. \frac{\partial \ell}{\partial \hat{y}} \right|_{\hat{y}_t,y_t}$
the derivative of the loss at time $t$ evaluated at the predicted
$\hat{y}_t$. Then $g_{ti} = g'_t x_{ti}$ and we can see that:
\begin{displaymath}
\begin{array}{ll}
\multicolumn{2}{l}{\sum_{i=1}^d \sqrt{S_{ii} \sum_{t=1}^T g^2_{ti}}}\\
= & d \sum_{i=1}^d \frac{1}{d} \sqrt{S_{ii} \sum_{t=1}^T g^2_{ti}}\\
\leq & d \sqrt{\sum_{i=1}^d \frac{1}{d} S_{ii} \sum_{t=1}^T g^2_{ti}}\\
= & \sqrt{d} \sqrt{\sum_{i=1}^d S_{ii} \sum_{t=1}^T g^2_{ti}}\\
= & \sqrt{d} \sqrt{\sum_{t=1}^T g'^2_t \sum_{i=1}^d S_{ii} x^2_{ti}}\\
\leq & \sqrt{d} \sqrt{\sum_{t=1}^T g'^2_t}\\
\end{array}
\end{displaymath}
where the last inequality holds by assumption.  
For $p=2$, we have $R_T \leq C \sqrt{d} \sqrt{\sum_{t=1}^T g'^2_t}$.

\subsection{Adaptive Conditioner}

Lemma \ref{lemma:minimaxhindsight} does not lead to a practical
algorithm, since at time $t$, we only observed $g_{1:t}$ and $x_{1:t}$,
when performing the update for $w_{t+1}$. Hence we would not be able
to compute this optimal conditioner $A^*$. However it suggests that we
could potentially approximate this ideal choice using the information
observed so far, e.g.,
\begin{equation}
A_{t,ii} = \frac{1}{C}  \sqrt{\frac{\sum_{s=1}^t g^2_{si}}{S^{(t)}_{ii}}},
\end{equation}
where $S^{(t)}$ is the diagonal matrix with minimum determinant
s.t. $||S^{1/2} x_i||_p \leq 1$ for all $i \in 1:t$.  There are two
potential sources of additional regret in the above choice, one from
truncating the sum of gradients, and the other from estimating the
enclosing volume online.

\subsubsection{Transductive Case}
To demonstrate that truncating the sum of gradients has only a modest
impact on regret we first consider the transductive case, i.e., we assume
we have access to all inputs $x_{1:T}$ that are coming in advance. However
at time $t$, we do not know the future gradients $g_{t+1:T}$. Hence for
this setting we could consider a 2-pass algorithm. On the first pass,
compute the diagonal matrix $S$, and then on the second pass, perform
adaptive gradient descent with the following conditioner at time $t$:
\begin{equation}\label{eqAdaptCondRangeTrans}
A_{t,ii} = \frac{1}{C \eta} \sqrt{\frac{\sum_{j=1}^t g^2_{ji}}{S_{ii}}}.
\end{equation}
We would like to be able to show that if we adapt the conditioner
in this way, than our regret is not much worse than with the best
conditioner in hindsight.  To do so, we must introduce a projection
step into the algorithm.  The projection step enables us to bound the
terms in lemma \ref{lemma:regretbound} corresponding to the use of a
non-constant conditioner, which are related to the maximum distance
between an intermediate weight vector and the optimal weight vector.

Define the projection $\Pi^{A}_{S,C,q}$ as
\begin{displaymath}
\Pi^{A}_{S,C,q}(w') = \argmin_{w \in \mathbb{R}^d | ||S^{-1/2} w||_q \leq C} (w-w')^{\top} A (w-w').
\end{displaymath}
Utilizing this projection step in the update we can show the following.
\begin{theorem}\label{thm:trans}
Let $S$ be the diagonal matrix with minimum determinant s.t. $||S^{1/2}
x_i||_p \leq 1$ for all $i \in 1:T$, and let $1/q = 1 - 1/p$.
If we choose $A_t$ as in Equation \ref{eqAdaptCondRangeTrans} with
$\eta=\sqrt{2}$ and use projection $w_{t+1} = \Pi^{A_t}_{S,C,q}(w_t -
A^{-1}_t g_t)$ at each step, the regret is bounded by
\begin{displaymath}
R_T \leq 2C \sqrt{2} \sum_{i=1}^d \sqrt{S_{ii} \sum_{j=1}^T g^2_{ji}}.
\end{displaymath}
\end{theorem}
We note that this is only a factor $2\sqrt{2}$ worse than when using
the best fixed $A$ in hindsight, knowing all gradients in advance.

\subsubsection{Streaming Case}
In this section we focus on the case $p = \infty$.

The transductive analysis indicates that using a partial sum of
gradients does not meaningfully degrade the regret bound.  We now
investigate the impact of estimating the enclosing ellipsoid \echo{with a
diagonal matrix} online using only observed inputs,
\begin{equation}\label{eqAdaptCondRangeStream}
A_{t,ii} = \frac{1}{C \eta} \sqrt{\sum_{j=1}^t g^2_{ji}} \max_{j \in
1:t} |x_{ji}|.
\end{equation}
\echo{The diagonal approximation is necessary for computational
  efficiency in NAG.}

Intuitively the worst case is when the conditioner in equation
\ref{eqAdaptCondRangeStream} differs substantially from the transductive
conditioner of equation \ref{eqAdaptCondRangeTrans} over most of the
sequence.  This is reflected in the regret bound below which is driven
by the ratio between the first non-zero value of an input $x_{ji}$
encountered in the sequence and the maximum value it obtains over the
sequence.
\begin{theorem}\label{thm:stream}
Let $p = \infty$, $q = 1$, and let $S^{(t)}$ be the diagonal matrix with
minimum determinant s.t. $||S^{1/2} x_i||_p \leq 1$ for all $i \in 1:t$.
Let $\Delta_i = \frac{\max_{t \in 1:T}|x_{ti}|}{|x_{t^i_0i}|}$,
for $t^i_0$ the first timestep the $i^{th}$ feature is non-zero.
If we choose $A_t$ as in Equation \ref{eqAdaptCondRangeStream}, $\eta =
\sqrt{2}$ and use projection $w_{t+1} = \Pi^{A_t}_{S^{(t)},C,q}(w_t -
A^{-1}_t g_t)$ at each step, the regret is bounded by
\begin{displaymath}
R_T \leq C \sum_{i=1}^d \frac{\sqrt{\sum_{j=1}^T g^2_{ji}}}{\max_{j\in1:T} |x_{ji}|} \left( \frac{1 + 6 \Delta_i + \Delta_i^2}{2 \sqrt{2}} \right).
\end{displaymath}
\end{theorem}
Comparing theorem \ref{thm:stream} with theorem \ref{thm:trans}
reveals the degradation in regret due to online estimation of the
enclosing ellipsoid.  Although an adversary can in general manipulate
this to cause large regret, there are nontrivial cases for which theorem
\ref{thm:stream} provides interesting protection.  For example, if the
non-zero feature values for dimension $i$ range over $[s_i, 2s_i]$ for
some unknown $s_i$, then $1 \leq \Delta_i \leq 2$ and the regret bound
is only a constant factor worse than the best choice of conditioner
in hindsight.

Because the worst case streaming scenario is when the initial sequence
has much lower scale than the entire sequence, we can improve the bound
if we weaken the ability of the adversary to choose the sequence order.
In particular, we allow the adversary to choose the sequence $\{ x_t, y_t
\}_{t=1}^T$ but then we subject the sequence to a random permutation
before processing it.  We can show that with high probability we
must observe a high percentile value of each feature after only a
few datapoints, which leads to the following corollary to theorem
\ref{thm:stream}.

\begin{corollary}\label{cor:permute}
Let $\{x_t, y_t\}_{t=1}^T$ be an exchangeable sequence with $x_t \in
\mathbb{R}^d$.  Let $p = \infty$, $q = 1$, and let $S^{(t)}$ be the
diagonal matrix with minimum determinant s.t. $||S^{1/2} x_i||_p \leq 1$
for all $i \in 1:t$.  Choose $\delta > 0$ and $\nu \in (0, 1)$.
Let $\Delta_i = \frac{\max_{t \in 1:T}|x_{ti}|}{\max_{t \in 1:\tau}
|x_{ti}|}$, where
\begin{displaymath}
\tau = \left\lceil \frac{\log(d/\delta)}{\nu} \right\rceil.
\end{displaymath}
If $R_{\max}$ is the maximum regret that can be incurred on a single
example, then choosing $\eta = \sqrt{2}$, and using projection $w_{t+1}
= \Pi^{A_t}_{S^{(t)},C,q}(w_t - A^{-1}_t g_t)$ at each step, the regret
is bounded by
\begin{displaymath}\begin{array}{rl}
R_T \leq & \left\lfloor \frac{\log(d/\delta)}{\nu} \right\rfloor R_{\max}\\
& + C \sum_{i=1}^d \frac{\sqrt{\sum_{j=1}^T g^2_{ji}}}{\max_{j\in1:T} |x_{ji}|} \left( \frac{1 + 6 \Delta_i + \Delta_i^2}{2\sqrt{2}} \right),
\end{array}
\end{displaymath}
and with probability at least $(1 - \delta)$ over sequence realizations,
for all $i \in 1:d$,
\begin{displaymath}
\Delta_i \leq \frac{\max_{t \in 1:T} |x_{ti}|}{\mathrm{Quantile\,} (\{ |x_{ti}| \}_{t=1}^T, 1-\nu)},
\end{displaymath}
where $\mathrm{Quantile} (\cdot, 1-\nu)$ is the $(1-\nu)$-th quantile
of a given sequence.
\end{corollary}
The quantity $R_{\max}$ can be related to $C$, if when making predictions,
we always truncate $w^\top_t x_t$ in the interval $[-C,C]$. For
instance, for the hinge loss and logistic loss,  $R_{\max} \leq C+1$
if we truncate our predictions this way. Similarly for the squared
loss, $R_{\max} \leq 4C\max(C,\max_{t} |y_t|)$. Although in theory an
adversary can manipulate the ratio between the maximum and an extreme
quantile to induce arbitrarily bad regret (i.e. make $\frac{\max_{t
\in 1:T} |x_{ti}|}{\mathrm{Quantile\,} (\{ |x_{ti}| \}_{t=1}^T,
1-\nu)}$ arbitrarily large even for small $\nu$), in practice we can
often expect this quantity to be close to 1\footnote{For instance, if
$|x_{ti}|$ are exponentially distributed, $\Delta_i$ is roughly less than
$\log(T/\delta)/\log(1/\nu)$ with probability at least $1-\delta$, thus
choosing $\nu = T^{-\alpha}$, for $\alpha \in (0,0.5]$ makes this a small
constant of order $\alpha^{-1} \log(1/\delta)$, while keeping the first
term involving $R_{\max}$ order of $T^\alpha \leq \sqrt{T}$.}, and thus
corollary \ref{cor:permute} suggests that we may perform not much worse
than when the scale of the features are known in advance. Our experiments
demonstrate that this is the common behavior of the algorithm in practice.

%

\section{Experiments}


\begin{table}
\begin{center}
\begin{tabular}{|c|c|c|c|}
Dataset & Size & Features & Scale Range \\ \hline
Bank & 45,212 & 7 & [ 31, 102127 ] \\
Census & 199,523 & 13 & [ 2, 99999 ] \\
Covertype & 581,012 & 54 & [ 1, 7173 ] \\
CT Slice & 53,500 & 360 & [ 0.158, 1 ] \\
MSD & 463,715 & 90 & [ 60, 65735 ] \\
Shuttle & 43,500 & 9 & [ 105, 13839 ] \\
\end{tabular}
\end{center}
\caption{Datasets used for experiments.  CT Slice and MSD are regression tasks, all others are classification.  The scale of a feature is defined as the maximum empirical absolute value, and the scale range of a dataset defined as the minimum and maximum feature scales. }
\label{tab:datasets}
\end{table}

\begin{table}
\begin{tabular}{l|c|c|c|c}
\multirow{2}{*}{Dataset} & \multicolumn{2}{|c|}{NAG} &
\multicolumn{2}{|c}{AG} \\  \cline{2-5}
& $\eta^*$ & Loss & $\eta^*$ & Loss  \\ \hline
\multirow{1}{*}{Bank} & 0.55 & \bf{0.098} & $5.5 * 10^{-5}$ & 0.109 \\
\multirow{1}{*}{(Maxnorm)}& 0.55 & 0.099 & 0.061 & 0.099 \\ \hline
\multirow{1}{*}{Census} & 0.2 & \bf{0.050} & $1.2 * 10^{-6}$ & 0.054 \\
\multirow{1}{*}{(Maxnorm)} & 0.25 & 0.050 & $8.3 * 10^{-3}$ & 0.051 \\ \hline
\multirow{1}{*}{Covertype} & 1.5 & \bf{0.27} & $5.6 * 10^{-7}$ & 0.32 \\
\multirow{1}{*}{(Maxnorm)}& 1.5 & 0.27 & 0.2 & 0.27 \\ \hline
\multirow{1}{*}{CT Slice} & 2.7 &  0.0023  & 0.022 &  0.0023 \\
\multirow{1}{*}{(Maxnorm)}& 2.7 & 0.0023 & 0.022 & 0.0023 \\ \hline
\multirow{1}{*}{MSD} & 9.0 & \bf{0.0110} & $5.5 * 10^{-7}$ & 0.0130 \\
\multirow{1}{*}{(Maxnorm)}& 9.0 & 0.0110 & 6.0 & 0.0108 \\ \hline
\multirow{1}{*}{Shuttle} & 7.4 & 0.036 & $7.5 * 10^{-4}$ & 0.040 \\
\multirow{1}{*}{(Maxnorm)} & 7.4 & 0.036 & 16.4 & 0.035 \\
\end{tabular}
\caption{Comparison of NAG with Adaptive Gradient (AG) across several
  data sets. For each data set, the first line in the table contains
  the results using the original data, and the second line contains
  the results using a max-norm pre-normalized version of the original data.
  For both algorithms, $\eta^*$ is the optimal in-hindsight learning
  rate for minimizing progressive validation loss (empirically
  determined).  Significance (bolding) was determined using a relative
  entropy Chernoff bound with a $0.1$ probability of bound failure.}
\label{tab:shootout}
\end{table}

\begin{table}
\begin{tabular}{l|c|c|c|c}
\multirow{2}{*}{Dataset} & \multicolumn{2}{|c|}{sNAG} &
\multicolumn{2}{|c}{AG} \\  \cline{2-5}
& $\eta^*$ & Loss & $\eta^*$ & Loss  \\ \hline
\multirow{1}{*}{Bank} & 0.3 & \bf{0.098} & $5.5 * 10^{-5}$ & 0.109 \\
\multirow{1}{*}{(Sq norm)}& 0.3 & 0.098 & 0.033 & 0.097 \\ \hline
\multirow{1}{*}{Census} & $0.11$ & \bf{0.050} & $1.2 * 10^{-6}$ & 0.054 \\
\multirow{1}{*}{(Sq norm)} & $0.11$ & 0.050 & $1.6 * 10^{-3}$ & 0.048 \\ \hline
\multirow{1}{*}{Covertype} & 2.2 & \bf{0.28} & $5.6 * 10^{-7}$ & 0.32 \\
\multirow{1}{*}{(Sq norm)}& 2.7 & 0.28 & 0.04 & 0.28 \\ \hline
\multirow{1}{*}{CT Slice} & 2.7 &  \bf{0.0019}  & 0.022 &  0.0023 \\
\multirow{1}{*}{(Sq norm)}& 2.7 & 0.0019 & 0.0067 & 0.0019 \\ \hline
\multirow{1}{*}{MSD} & 7.4 & 0.0119 & $5.5 * 10^{-7}$ & 0.0130 \\
\multirow{1}{*}{(Sq norm)}& 7.4 & 0.0118 & 0.05 & 0.0120 \\ \hline
\multirow{1}{*}{Shuttle} & 11 & \bf{0.026} & $7.5 * 10^{-4}$ & 0.040 \\
\multirow{1}{*}{(Sq norm)} & 9 & 0.026 & 0.818 & 0.026 \\
\end{tabular}
\caption{(Online Mean Square Normalized) Comparison of NAG with
  Adaptive Gradient (AG) across several data sets. For each data set,
  the first line in the table contains the results using the original
  data, and the second line contains the results using a squared-norm
  pre-normalized version of the original data.  For both algorithms,
  $\eta^*$ is the optimal in-hindsight learning rate for minimizing
  progressive validation loss (empirically determined).  Significance
  (bolding) was determined using a relative entropy Chernoff bound
  with a $0.1$ probability of bound failure.}
\label{tab:shootout-squared-N}
\end{table}

Table \ref{tab:shootout} compares a variant of the normalized learning
rule to the adaptive gradient method~\cite{MS10,DHS11} with $p = 2$
and without projection step for both algorithms.  For each data set we
exhaustively searched the space of learning rates to optimize average
progressive validation loss.  Besides the learning rate, the learning rule
was the only parameter adjusted between the two conditions.  The loss
function used depended upon the task associated with the dataset,
which was either 0-1 loss for classification tasks or squared loss for 
regression tasks.  For regression tasks, the loss is divided by the worst 
possible squared loss, i.e., $(\max - \min)^2$.

The datasets utilized are: \emph{Bank}, the
UCI~\cite{FrankAsuncion2010} Bank Marketing Data Set~\cite{moro2011};
\emph{Census}, the UCI Census-Income KDD Data Set; \emph{Covertype},
the UCI Covertype Data Set; \emph{CT Slice}, the UCI Relative
Location of CT Slices on Axial Axis Data Set; \emph{MSD}, the
Million Song Database~\cite{Bertin-Mahieux2011}; and \emph{Shuttle},
the UCI Statlog Shuttle Data Set.  These were selected as public
datasets plausibly exhibiting varying scales or lack of normalization.
On other pre-normalized datasets which are publicly available, we observed
relatively little difference between these update rules.  To demonstrate
the effect of pre-normalization on these data sets, we constructed a 
pre-normalized version of each one by dividing every feature by its maximum empirical
absolute value.

\begin{figure}
\begin{center}
\includegraphics[angle=270,scale=0.7]{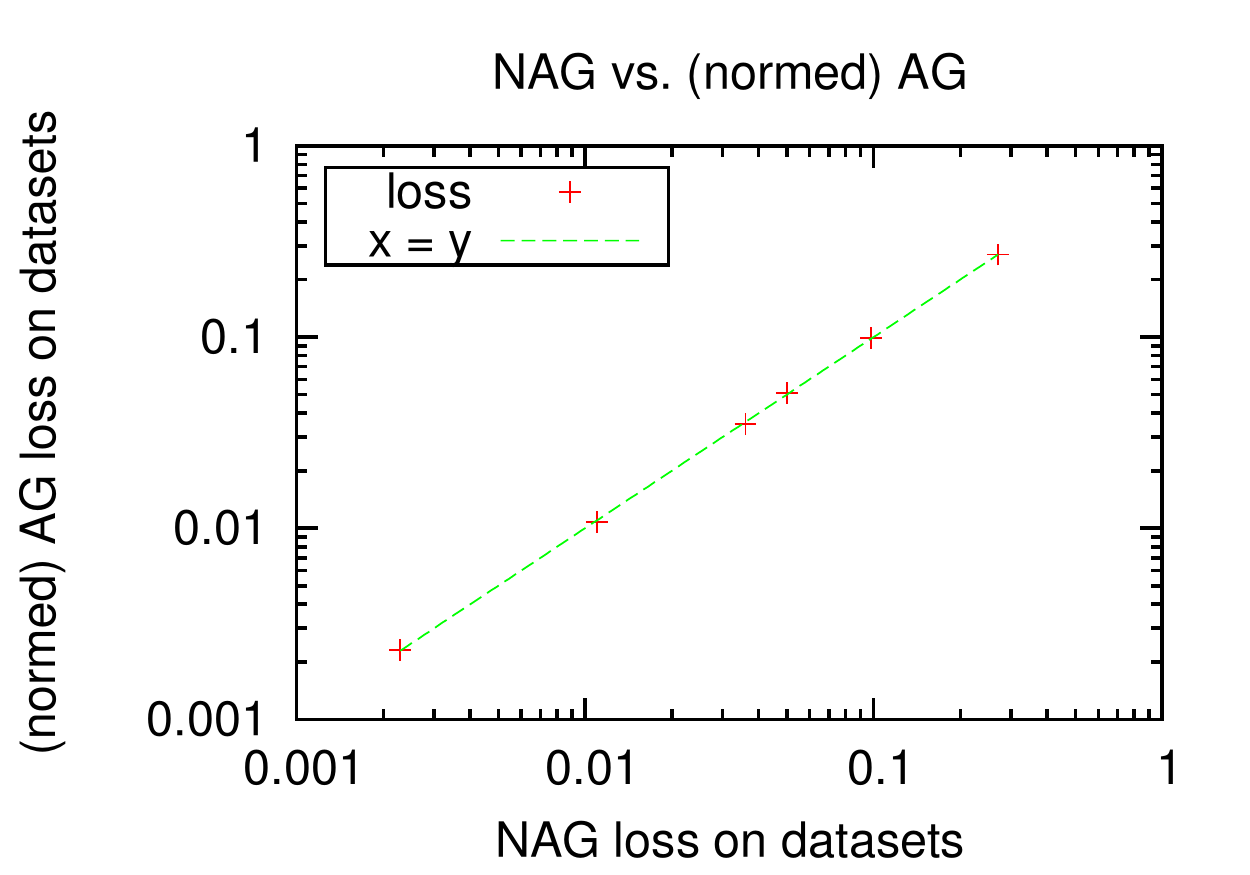}
\end{center}
\vspace{-16pt}
\caption{A comparison of performance of NAG and pre-normed AG.  The results are identical, indicating that NAG effectively obsoletes pre-normalization.}
\label{fig:perf}
\end{figure}

Some trends are evident from table \ref{tab:shootout}.  First, the
normalized learning rule (as expected) has highest impact when the
individual feature scales are highly disparate, such as data assembled
from heterogeneous sensors or measurements.  For instance, the CT
slice data set exhibits essentially no difference; although CT slice
contains physical measurements, they are histograms of raw readings
from a single device, so the differences between feature ranges is
modest (see table \ref{tab:datasets}). Conversely
the Covertype dataset shows a 5\% decrease in multiclass 0-1 loss over
the course of training.  Covertype contains some measurements in
units of meters and others in degrees, several ``hillshade index''
values that range from 0 to 255, and categorical variables.

The second trend evident from table \ref{tab:shootout} and reproduced
in figure \ref{fig:etarange} is that the optimal learning rate is both
closer to 1 in absolute terms, and varies less in relative terms,
between data sets.  This substantially eases the burden of tuning the
learning rate for high performance.  For example, a randomized
search~\cite{BB12} is much easier to conduct given that the optimal
value is extremely likely to be within $[0.01, 10]$ independent of the
data set.

\begin{figure}
\begin{center}
\includegraphics*[viewport=0 150 220 350,height=2.5in,width=3in]{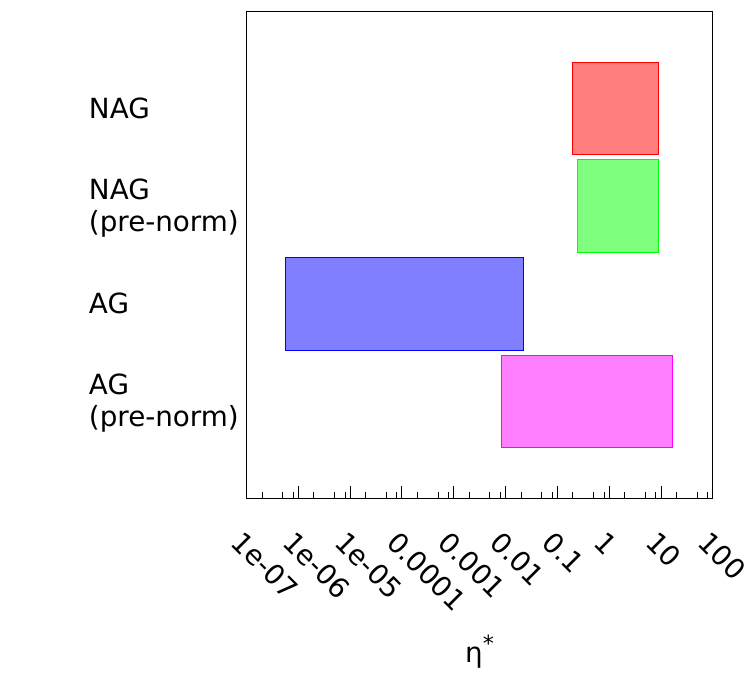}
\end{center}
\vspace{-16pt}
\caption{\echo{Each color represents the} range of optimal in-hindsight learning rates $\eta^*$ across the datasets for the different learning algorithms.  NAG exhibits a much smaller range of optimal values even when the data sets are pre-normalized, easing the problem of hyperparameter selection.}
\label{fig:etarange}
\end{figure}

The last trend evident from table \ref{tab:shootout} is the typical
indifference of the normalized learning rate to pre-normalization,
specifically the optimal learning rate and resulting progressive
validation loss.  In addition pre-normalization effectively eliminates
the difference between the normalized and Adagrad updates, 
indicating that the online algorithm achieves results similar
to the transductive algorithm for max norm.

For comprehensiveness, we also compared sNAG with a squared norm
pre-normalizer, and found the story much the same in
table~\ref{tab:shootout-squared-N}.  In particular sNAG dominated AG
on most of the datasets and performed similarly to AG when the data
was pre-normalized with a squared norm (standard deviation).  It is
also interesting to observe that sNAG performs slightly better than
NAG on a few datasets, agreeing with our intuition that it should be
more robust to outliers.  \echo{Empirically, sNAG appears somewhat more
robust than NAG at the cost of somewhat more computation.}

\section{Summary}

We evaluated performance of Normalized Adaptive Gradient (NAG) on the
most difficult unnormalized public datasets available and found that
it provided performance similar to Adaptive Gradient (AG) applied to
pre-normalized datasets while simultaneously collapsing the range
hyperparameter search required to achieve good performance.
Empirically, this makes NAG a capable and reliable learning algorithm.

We also defined a scaling adversary and proved that our algorithm is
robust and efficient against this scaling adversary unlike other
online learning algorithms.

\section*{Acknowledgements}

We would like to thank Miroslav Dudik for helpful discussions.

\bibliographystyle{apalike}
\bibliography{nonline}

\newpage

\section{Appendix (Proofs)}

\subsection{Proof of Lemma \ref{lemma:regretbound}}

\begin{proof}
\begin{displaymath}
\begin{array}{ll}
\multicolumn{2}{l}{(w_{t+1} - w^*)^\top A_t(w_{t+1} - w^*)}\\
= & (w_{t} - A_t^{-1}g_t - w^*)^\top A_t(w_{t} - A_t^{-1}g_t - w^*)\\
= & (w_{t} - w^*)^\top A_t(w_{t} - w^*) - 2 g^{\top}_t (w_t - w^*) +
g^{\top}_t A_t^{-1} g_t,\\
\end{array}
\end{displaymath}
which implies $g^\top_t (w_t - w^*) = \frac{1}{2} ( (w_{t} - w^*)^\top A_t(w_{t} - w^*) - (w_{t+1} - w^*)^\top A_t(w_{t+1} - w^*) + g^\top_t A^{-1}_t g_t ).$

Next by convexity of the loss, $\ell(w_t^{\top} x_t,y_t) - \ell(w^{*\top} x_t,y_t) \leq g_t^{\top} (w_t - w^*)$.  Therefore
\begin{displaymath}
\begin{array}{ll}
\multicolumn{2}{l}{2 R_T}\\
\leq & 2 \sum_{t=1}^T g_t^{\top} (w_{t} - w^*)\\
= & \sum_{t=1}^T \left( (w_{t} - w^*)^\top A_t(w_{t} - w^*) \right. \\
   & \left. - (w_{t+1} - w^*)^\top A_t(w_{t+1} - w^*) + g^\top_t A_t g_t \right) \\
= &  (w_{1} - w^*)^\top A_1(w_{1} - w^*) \\
    & - (w_{T+1} - w^*)^\top A_T(w_{T+1} - w^*) \\
    & + \sum_{t=1}^{T-1} (w_{t+1} - w^*)^\top (A_{t+1} - A_t)(w_{t+1} - w^*) \\
    & + \sum_{t=1}^T g^\top_t A_t g_t \\
\leq & (w_{1} - w^*)^\top A_1(w_{1} - w^*) \\
       & + \sum_{t=1}^{T-1} (w_{t+1} - w^*)^\top (A_{t+1} - A_t)(w_{t+1} - w^*) \\
       & + \sum_{t=1}^T g^\top_t A_t g_t. \\
\end{array}
\end{displaymath}
\end{proof}

\subsubsection{Proof of Lemma \ref{lemma:minimaxhindsight}}

The following result will be useful.
\begin{lemma} \label{lemAdaptiveGrad}
For any vector $x \in \mathcal{R}^n$,
\begin{displaymath}
\sum_{i=1}^n \frac{x_i^2}{\sqrt{\sum_{j=1}^i x_j^2}} \leq 2 \sqrt{\sum_{i=1}^n x_i^2}.
\end{displaymath}
\end{lemma}
\begin{proof}
Proof by induction. Clearly for $n=1$ this statement is true, as
$x_1^2/\sqrt{x_1^2} = \vert x_1 \vert \leq 2 \sqrt{x_1^2}$. Now suppose
it is true for $n=k$, we will show it must be true for $n=k+1$. We have:
\begin{displaymath}
\begin{array}{ll}
\multicolumn{2}{l}{\sum_{i=1}^{k+1} \frac{x_i^2}{\sqrt{\sum_{j=1}^i x_j^2}}}\\
= & \sum_{i=1}^{k} \frac{x_i^2}{\sqrt{\sum_{j=1}^i x_j^2}} + \frac{x_{k+1}^2}{\sqrt{\sum_{j=1}^{k+1} x_j^2}}\\
\leq & 2 \sqrt{\sum_{j=1}^k x_j^2} + \frac{x_{k+1}^2}{\sqrt{\sum_{j=1}^{k+1} x_j^2}}\\
= & 2 \sqrt{\sum_{j=1}^{k+1} x_j^2 - x_{k+1}^2} + \frac{x_{k+1}^2}{\sqrt{\sum_{j=1}^{k+1} x_j^2}}\\
\leq & 2 \sqrt{\sum_{j=1}^{k+1} x_j^2} - \frac{x_{k+1}^2}{\sqrt{\sum_{j=1}^{k+1} x_j^2}} + \frac{x_{k+1}^2}{\sqrt{\sum_{j=1}^{k+1} x_j^2}}\\
= & 2 \sqrt{\sum_{j=1}^{k+1} x_j^2},\\
\end{array}
\end{displaymath}
where the first inequality follows from our induction hypothesis,
and the last inequality uses the fact that since the square root is
concave, it is upper bounded by its first order Taylor series expansion:
i.e. $\sqrt{x-a} \leq \sqrt{x} - \frac{a}{2\sqrt{x}}$ for any $a \leq x$.
\end{proof}
The proof of lemma \ref{lemma:minimaxhindsight} now follows.
\begin{proof}
We seek to choose $A$ to minimize
\begin{displaymath}
\min_{A} \max_{w^* \in \mathcal{W}^C_X} \frac{1}{2} \sum_{i=1}^d \left( A_{ii} (w^*_i)^2 + \sum_{t=1}^T \frac{g^2_{ti}}{A_{ii}} \right).
\end{displaymath}
We first solve for the max given $A$. Consider the problem $\max_{w \in
\mathcal{W}^C_X} w^\top A w$ for the case where we define $\mathcal{X}$
using $p=2$. By doing a change of variable $y = C^{-1} S^{-1/2} w$,
this can be rewritten as:
\begin{displaymath}
\max_{y \mid y^\top y \leq 1} C^2 y^\top S^{1/2} A S^{1/2} y,
\end{displaymath}
assuming $S$, the matrix defining the ellipsoid of the input space is
full rank (invertible). To maximize this, one simply chooses $y$ to be
in the direction of the eigenvector of $S^{1/2} A S^{1/2}$ with maximum
eigenvalue:
\begin{displaymath}
\max_{w \in \mathcal{W}^C_X} w^\top A w = C^2 \lambda_{\max}(S^{1/2} A S^{1/2}),
\end{displaymath}
where $\lambda_{\max}(M)$ denotes the maximum eigenvalue of matrix
$M$. This can be upperbounded using the trace:
\begin{displaymath}
\begin{array}{ll}
\multicolumn{2}{l}{\max_{w \in \mathcal{W}^C_X} w^\top A w}\\
= & C^2 \lambda_{\max}(S^{1/2} A S^{1/2})\\
\leq & C^2 \tr(S^{1/2} A S^{1/2}) \\
= & C^2 \tr(A S).\\
\end{array}
\end{displaymath}
When $A$ and $S$ are diagonal, it can also be seen that for any $p$ norm
used to define the set of input $\mathcal{X}$, this quantity is bounded
by $C^2 \tr(A S)$. This is because using the same change of variable $z =
C^{-1} S^{-1/2} w$, the optimization can be rewritten as:
\begin{displaymath}
\max_{z \mid ||z||_q \leq 1} C^2 z^\top S^{1/2} A S^{1/2} z.
\end{displaymath}
Since $\{z \mid ||z||_q \leq 1 \} \subseteq \{z \mid ||z||_\infty \leq 1 \}$, implying:
\begin{displaymath}
\begin{split}
&\max_{z \mid ||z||_q \leq 1} C^2 z^\top S^{1/2} A S^{1/2} z 
\leq \max_{z \mid ||z||_\infty \leq 1} C^2 z^\top S^{1/2} A S^{1/2} z.
\end{split}
\end{displaymath}
Now since $A$ and $S$ are diagonal we have:
\begin{displaymath}
C^2 z^\top S^{1/2} A S^{1/2} z = C^2 \sum_{i=1}^d z_i^2 A_{ii} S_{ii},
\end{displaymath}
and since $|z_i| \leq 1$ for all $i$, then we obtain $C^2 z^\top S^{1/2}
A S^{1/2} z \leq C^2 \sum_{i=1}^d A_{ii} S_{ii} = C^2 \tr(AS)$.

If we use this upper bound for all $p$, then when $A$ and $S$ are
diagonal, the regret is bounded by:
\begin{displaymath}
R_T \leq \frac{1}{2} \sum_{i=1}^d \left( C^2 A_{ii} S_{ii} + \sum_{t=1}^T \frac{g^2_{ti}}{A_{ii}} \right).
\end{displaymath}
Taking the derivative w.r.t. $A_{ii}$ and solving for 0:
\begin{displaymath}
A^*_{ii} = \frac{1}{C} \sqrt{\frac{\sum_{t=1}^T g^2_{ti}}{S_{ii}}}.
\end{displaymath}
Substituting this choice of $A$ into the regret bound:
\begin{displaymath}
R_T \leq C \sum_{i=1}^d \sqrt{S_{ii} \sum_{t=1}^T g^2_{ti}}.
\end{displaymath}
\end{proof}

\subsection{Proof of Theorem \ref{thm:trans}}

\begin{proof}
To use lemma \ref{lemma:regretbound}, it suffices to show the projection
step satisfies $(w_{t+1} - w^*)^\top A_t(w_{t+1} - w^*) \leq (w_{t} -
A^{-1}_t g_t - w^*)^\top A_t(w_{t} - A^{-1}_t g_t - w^*)$.  The projection
step for Theorem 1 projects onto the set containing $w^*$ under norm $A_t$
and thus guarantees this directly.

The bound of lemma \ref{lemma:regretbound} has two terms in it.
Applying lemma \ref{lemAdaptiveGrad} directly to the second term tells
us that $\sum_{t=1}^T g^\top_t A^{-1}_t g_t \leq 2 \eta  C \sum_{i=1}^d
\sqrt{S_{ii} \sum_{t=1}^T g^2_{ti}}$.

For the first term, since $A_{t+1,ii} \geq A_{t,ii}$ for all $t,i$, then
\begin{displaymath}
\begin{split}
&(w_{1} - w^*)^\top A_1(w_{1} - w^*) \\
&+ \sum_{t=1}^{T-1} (w_{t+1} - w^*)^\top (A_{t+1} - A_t)(w_{t+1} - w^*
) \\
&\leq \sum_{i=1}^d A_{T,ii} \sup_t (w_{ti} - w^*_{i})^2.
\end{split}
\end{displaymath}
The projection guarantees that the maximum value $|w_{ti}|$ could take
is $C \sqrt{S_{ii}}$ for any time $t$, and similarly for $|w^*_i|$,
implying $\sup_t (w_{ti} - w^*_{i})^2 \leq 4 C^2 S_{ii}$.  Thus
\begin{displaymath}
\sum_{i=1}^d A_{T,ii} \sup_t (w_{ti} - w^*_{i})^2 \leq \frac{4 C}{\eta}
\sum_{i=1}^d \sqrt{S_{ii} \sum_{j=1}^T g^2_{ji}}.
\end{displaymath}
Adding the bounds and choosing $\eta = \sqrt{2}$ yields the desired
result.
\end{proof}

\subsection{Proof of Theorem \ref{thm:stream}}

\begin{proof}
To use lemma \ref{lemma:regretbound}, it suffices to show the projection
step satisfies $(w_{t+1} - w^*)^\top A_t(w_{t+1} - w^*) \leq (w_{t} -
A^{-1}_t g_t - w^*)^\top A_t(w_{t} - A^{-1}_t g_t - w^*)$.  For this
in turn it suffices to show that $w^*$ is in the set onto which
the projection is done.  This is guaranteed by the definition of
$S^{(t)}$, which implies the input cover is strictly increasing so that
$\mathcal{X}_t \subseteq \mathcal{X}$, and therefore the dual cover must
always contain $W_X^C$.

The bound of lemma \ref{lemma:regretbound} has two terms in it.
To bound the second term let $t_0^i$ denote the first timestamp where
$x_i \neq 0$.  Then
\begin{displaymath}
\begin{array}{ll}
\multicolumn{2}{l}{\sum_{t=1}^T g^\top_t A^{-1}_t g_t}\\
= & C \eta \sum_{i=1}^d \sum_{t=t^i_0}^T
\frac{g^2_{ti}}{\sqrt{\sum_{j=1}^t g^2_{ji}} \max_{j\in1:t} |x_{ji}|} \\
= & C \eta \sum_{i=1}^d \sum_{t=t^i_0}^T \frac{\max_{j\in1:T} |x_{ji}|}{\max_{j\in1:t} |x_{ji}|} \frac{g^2_{ti}}{\sqrt{\sum_{j=1}^t g^2_{ji}} \max_{j\in1:T} |x_{ji}|}\\
\leq & 2 C \eta \sum_{i=1}^d \frac{\max_{j\in1:T} |x_{ji}|}{|x_{t^i_0i}|}
\frac{\sqrt{\sum_{j=1}^T g^2_{ji}}}{\max_{j\in1:T} |x_{ji}|}, \\
\end{array}
\end{displaymath}
where the last inequality uses lemma \ref{lemAdaptiveGrad} and the
strictly increasing property of $\max_{j \in 1:t} |x_{ji}|$.

For the first term, since $A_{t+1,ii} \geq A_{t,ii}$ for all $t,i$, then
\begin{displaymath}
\begin{split}
&(w_{1} - w^*)^\top A_1(w_{1} - w^*) \\
&+ \sum_{t=1}^{T-1} (w_{t+1} - w^*)^\top (A_{t+1} - A_t)(w_{t+1} - w^*
) \\
&\leq \sum_{i=1}^d A_{T,ii} \sup_t (w_{ti} - w^*_{i})^2.
\end{split}
\end{displaymath}
The projection step guarantees that the maximum value $|w_{ti}|$ could
take is $C / |x_{t_0^i,i}|$, whereas the maximum absolute value of $w^*_i$
is $\frac{C}{\max_{t \in 1:T}|x_{ti}|}$.  Therefore
\begin{displaymath}
\begin{array}{ll}
\multicolumn{2}{l}{\sup_t (w_{ti} - w^*_{i})^2}\\
\leq & C^2 (\frac{1}{|x_{t^i_0i}|} + \frac{1}{\max_{t \in 1:T}|x_{ti}|})^2\\
= & \frac{C^2}{\max_{t \in 1:T}|x_{ti}|^2} (\frac{\max_{t \in 1:T}|x_{ti}|}{|x_{t^i_0i}|} + 1)^2,
\end{array}
\end{displaymath}
and we obtain that
\begin{displaymath}
\begin{split}
&\sum_{i=1}^d A_{T,ii} \sup_t (w_{ti} - w^*_{i})^2 \\
&\leq \frac{C}{\eta} \sum_{i=1}^d \frac{\sqrt{\sum_{j=1}^T g^2_{ji}}}{\max_{j\in1:T} |x_{ji}|} \left( \frac{\max_{t \in 1:T}|x_{ti}|}{|x_{t^i_0i}|} + 1 \right)^2.
\end{split}
\end{displaymath}
Adding the bounds and choosing $\eta = \sqrt{2}$ yields the desired
result.
\end{proof}

\subsection{Proof of Corollary \ref{cor:permute}}
\begin{proof}
To prove this corollary, we first prove a lemma which will be useful. Let
$x_1,x_2,\dots,x_T$ be a sequence of data points in $\mathbb{R}^d$,
let $x^{(1)}_{i},x^{(2)}_{i},\dots,x^{(T)}_{i}$ denote the values of the
$i^{th}$ feature in decreasing order (i.e. $x^{(j)}_i$ is the $j^{th}$
largest value of feature $i$), and for $p \in (0,1]$, we denote the
$1-p$ percentile of feature $i$ as $Quantile(\{x_{ti}\}_{t=1}^T,1-p) =
x^{(\lceil Tp \rceil)}_i$.

\begin{lemma}\label{lemma:percentile}
For any $\delta, p \in (0,1)$, if $x_1,x_2,\dots,x_T$ are exchangeable
then with probability at least $1-\delta$, we must have observed a value
at least greater or equal to $Quantile(\{x_{ti}\}_{t=1}^T,1-p)$ for all
features $i$ after $t = \left\lceil \frac{\log(d/\delta)}{p} \right\rceil$ examples.
\end{lemma}
\begin{proof}
Consider the case where there is only one feature. Since the sequence is
exchangeable, the probability that we do not observe a value above the
$(1-p)$ percentile of that feature after $n$ datapoints is less than or
equal to $(1-p)^n$ (equal in the limit as $T \rightarrow \infty$). Now
with $d$ features, we can use a union bound over all features, so that
the probability that do not observe a value above the $(1-p)$ percentile
for at least 1 feature out of the $d$ features after n datapoints is
bounded by $d(1-p)^n \leq d \exp(-pn)$. $d \exp(-pn) \leq \delta$ when
$n \geq \frac{\log(d/\delta)}{p}$.
\end{proof}

Now to prove the corollary, consider bounding the regret in the first
$\tau-1$ steps by $R_{\max}$, and carrying a similar proof to Lemma
\ref{lemma:regretbound} to bound the regret in the remaining steps. We
obtain that for any $\tau \in \{1,2,\dots,T\}$:
\begin{displaymath} \begin{array}{rl}
R_T \leq & (\tau-1) R_{\max} + \frac{1}{2} [(w_{\tau} - w^*)^\top A_{\tau}(w_{\tau} - w^*) \\
& + \sum_{t=\tau}^{T-1} (w_{t+1} - w^*)^\top (A_{t+1} - A_t)(w_{t+1} - w^*)  \\
& + \sum_{t=\tau}^T g^\top_t A^{-1}_t g_t ]
\end{array}
\end{displaymath}
Let $\Delta^{(\tau)}_{i} = \frac{\max_{j\in1:T} |x_{ji}|}{\max_{j
\in 1:\tau}|x_{ji}|}$. Then by following a similar proof to Theorem
\ref{thm:stream}, we have that:
\begin{displaymath}
\sum_{t=\tau}^T g^\top_t A^{-1}_t g_t \leq 2 C \eta \sum_{i=1}^d \Delta^{(\tau)}_{i} \frac{\sqrt{\sum_{j=1}^T g^2_{ji}}}{\max_{j\in1:T} |x_{ji}|}
\end{displaymath}
and that:
\begin{displaymath}
\begin{array}{l}
(w_{\tau} - w^*)^\top A_{\tau} (w_{\tau} - w^*) \\
+ \sum_{t=\tau}^{T-1} (w_{t+1} - w^*)^\top (A_{t+1} - A_t)(w_{t+1} - w^* ) \\
\leq \frac{C}{\eta} \sum_{i=1}^d \frac{\sqrt{\sum_{j=1}^T g^2_{ji}}}{\max_{j\in1:T} |x_{ji}|} \left(\Delta^{(\tau)}_{i} + 1 \right)^2
 \end{array}
\end{displaymath}
Thus we have that:
\begin{displaymath}\begin{array}{rl}
R_T \leq & (\tau-1) R_{\max} \\
& + \frac{C}{2\eta} \sum_{i=1}^d \frac{\sqrt{\sum_{j=1}^T g^2_{ji}}}{\max_{j\in1:T} |x_{ji}|} \left(\Delta^{(\tau)}_{i} + 1 \right)^2 \\
& + C \eta \sum_{i=1}^d \Delta^{(\tau)}_{i} \frac{\sqrt{\sum_{j=1}^T g^2_{ji}}}{\max_{j\in1:T} |x_{ji}|}
\end{array}
\end{displaymath}
Now using lemma \ref{lemma:percentile}, we must have that for $\tau =
\lceil \frac{\log(d/\delta)}{p} \rceil$, $\max_{j \in 1:\tau}|x_{ji}|
\geq Quantile(\{|x_{ti}|\}_{t=1}^T,1-p)$ for all $i$ with probability
at least $1-\delta$. This proves the corollary.
\end{proof}

\end{document}